%% file: techrep_after_reviews.tex
\documentclass[letterpaper, 10 pt, conference]{ieeeconf}  

\IEEEoverridecommandlockouts                              
\usepackage{amsmath}
\usepackage{amssymb}
\usepackage{color}
\usepackage[usenames,dvipsnames]{xcolor}
\usepackage{graphicx,subfig}

\usepackage{caption}
\usepackage{algorithm}
\usepackage{algorithmic}

\usepackage{multirow}

\newcommand{\dts}{\mathcal{T}}
\newcommand{\ap}{AP}
\newcommand{\strat}{C}
\newcommand{\ba}{\mathcal{B}}
\newcommand{\product}{\mathcal{P}}
\newcommand{\scc}{\mathcal{U}}

\newcommand{\SCC}{\mathrm{SCC}}
\newcommand{\ASCC}{\mathrm{ASCC}}

\newcommand{\runinf}{\mathrm{inf}}
\newcommand{\run}{\mathrm{Run}}
\newcommand{\runfin}{\mathrm{Run}_\mathrm{fin}}

\newcommand{\nat}{\mathbb{N}}
\newcommand{\G}{\mathbf{G}}
\newcommand{\F}{\mathbf{F}}
\newcommand{\X}{\mathbf{X}}
\newcommand{\U}{\mathbf{U}}

\newcommand{\cyc}{\mathsf{cyc}}

\newcommand{\pen}{g}
\newcommand{\penexp}{\pen_{\mathrm{exp}}}
\newcommand{\pensim}{\pen_{\mathrm{sim}}}

\newcommand{\vis}{v}
\newcommand{\Vis}{\mathrm{Vis}}
\newcommand{\pisur}{\pi_{\mathrm{sur}}}
\newcommand{\sur}{\mathrm{sur}}

\newcommand{\cycles}{\sharp}

\newcommand{\stratoff}{\strat_{\product}}
\newcommand{\straton}{\mathbf{\strat}_{\product}}

\newcommand{\pr}{\mathrm{Pr}}

\newcommand{\seq}{\mathrm{run}}

\newcommand{\intdiv}{\,\,\, \mathrm{div}\,}
\newcommand{\intmod}{\,\,\, \mathrm{mod}\,}

\newcommand{\ie}{{\it i.e., }}
\newcommand{\eg}{{\it e.g., }}

\newtheorem{remark}{Remark}
\newtheorem{definition}{Definition}
\newtheorem{problem}{Problem}
\newtheorem{theorem}{Theorem}
\newtheorem{proposition}{Proposition}

\title{\LARGE \bf
Optimal Receding Horizon Control for Finite Deterministic Systems \\ with Temporal Logic Constraints
}

\author{M\'{a}ria Svore\v{n}ov\'{a}, Ivana \v{C}ern\'{a} and Calin Belta
\thanks{M. Svore\v{n}ov\'{a}, I. \v{C}ern\'{a} are with Faculty of Informatics, Masaryk University, Brno, Czech Republic, {\tt\footnotesize svorenova@mail.muni.cz, cerna@muni.cz}. C. Belta is with Department of Mechanical Engineering and the Division of Systems Engineering, Boston University, Boston, MA, USA, {\tt\footnotesize cbelta@bu.edu}. This work was partially supported at MU by grants GAP202/11/0312, LH11065, at BU by ONR grants MURI N00014-09-1051, MURI N00014-10-10952 and by NSF grant CNS-1035588.} 
}

\begin{document}

\maketitle
\thispagestyle{empty}
\pagestyle{empty}


\begin{abstract}
In this paper, we develop a provably correct optimal control strategy for a finite deterministic transition system. By assuming that penalties with known probabilities of occurrence and dynamics can be sensed locally at the states of the system, we derive a receding horizon strategy that minimizes the expected average cumulative penalty incurred between two consecutive satisfactions of a desired property. At the same time, we guarantee the satisfaction of correctness specifications expressed as Linear Temporal Logic formulas. We illustrate the approach with a persistent surveillance robotics application.
\end{abstract}


\section{Introduction}

Temporal logics, such as Computation Tree Logic (CTL) and Linear Temporal Logic (LTL), have been customarily used to specify the correctness of computer programs and digital circuits modeled as finite-state transition systems \cite{baierBook}. The problem of analyzing such a model against a temporal logic formula, known as formal analysis or model checking, has received a lot of attention during the past thirty years, and several efficient algorithms and software tools are available \cite{DiVinE, PRISM, NuSMV}. The formal synthesis problem, in which the goal is to design or control a system from a temporal logic specification, has not been studied extensively until a few years ago. Recent results include the use of model checking algorithms for controlling deterministic systems \cite{janasteve11}, automata games for controlling non-deterministic systems \cite{chatterjeeNTSpo}, linear programming and value iteration for synthesis of control policies for Markov decision processes \cite{baierBook,dennismdpCDC11}. Through the use of abstractions, such techniques have also been used for infinite systems, such as continuous and discrete-time linear systems \cite{Tabuada-TAC,Ebru-HSCC,Nok,Marius-LTL-TAC,Jana-Boyan-TAC}.     

The connection between optimal and temporal logic control is an intriguing problem with a potentially high impact in several applications. By combining these two seemingly unrelated areas, our goal is to optimize the behavior of a system subject to correctness constraints. Consider, for example, a mobile robot involved in a persistent surveillance mission in a dangerous area and under tight fuel / time constraints. The correctness requirement is expressed as a temporal logic specification, \eg ``Alternately keep visiting $A$ and $B$ while always avoiding $C$", while the resource constraints translate to minimizing a cost function over the feasible trajectories of the robot. While optimal control is a mature discipline and formal synthesis is fairly well understood, optimal formal synthesis is a largely open area. 

In this paper, we focus on finite labeled transition systems and correctness specifications given as formulas of LTL. We assume there is a penalty associated with the states of the system with a known occurrence probability and time-behavior. Motivated by persistent surveillance robotic missions, our goal is to minimize the expected average cumulative penalty incurred between two consecutive satisfactions of a desired property associated with some states of the system, while at the same time satisfying an additional temporal logic constraint. Also from robotics comes our assumption that actual penalty values can only be sensed locally in a close proximity from the current state during the execution of the system. We propose two algorithms for this problem. The first operates offline, \ie without executions of the system, and therefore only uses the known probabilities but does not exploit actual 
penalties sensed during the execution. The second algorithm designs an online 
strategy by locally improving the offline strategy based on local sensing and simulation over a user-defined planning horizon. While both algorithms guarantee optimal expected average penalty collection, in real executions of the system, the second algorithm provides lower real average than the first algorithm. We illustrate these results on a robotic persistent surveillance case study.  

This paper is closely related to \cite{dennisACC12,majajanaCDC12,janasteve11}, which also focused on optimal control for finite transitions systems with temporal logic constraints. In \cite{janasteve11}, the authors developed an offline control strategy minimizing the maximum cost between two consecutive visits to a given set of states, subject to constraints expressed as LTL formulas. Time-varying, locally sensed rewards were introduced in \cite{dennisACC12}, where a receding horizon control strategy maximizing rewards collected locally was shown to satisfy an LTL specification.
This approach was 
generalized in \cite{majajanaCDC12} to allow for a broader class of optimization objectives and reward models. In contrast with \cite{dennisACC12,majajanaCDC12}, we interpret the dynamic values appearing in states of the system as penalties instead of rewards, \ie in our case, the cost function is being minimized rather than maximized. 
That allows the existence of the optimum in expected average penalty collection. In this paper, we show how it can be achieved using automata-based approach and game theory results.

In Sec.~\ref{sec:prelim}, we introduce the notation and definitions necessary throughout the paper. The problem is stated in Sec.~\ref{sec:pf}.
The main results of the paper are in Sec.~\ref{sec:solution} and Sec.~\ref{sec:discussion}. The simulation results are presented in Sec.~\ref{sec:casestudy}.

%
%
%


\section{Preliminaries}
\label{sec:prelim}

For a set $\mathsf{S}$, we use $\mathsf{S}^{\omega}$ and $\mathsf{S}^+$ to denote the set of all infinite and all non-empty finite sequences of elements of $\mathsf{S}$, respectively. For a finite or infinite sequence $\alpha=\mathsf{a_0a_1}\ldots$, we use $\alpha(i)=\mathsf{a_i}$ to denote the $i$-th element 
and $\alpha^{(i)}=\mathsf{a_0\ldots a_i}$ for the finite prefix of $\alpha$ of length $|\alpha^{(i)}|=i+1$.



\begin{definition}
A weighted deterministic transition system (TS) is a tuple $\dts=(S,T,\ap,L,w)$, where $S$ is a non-empty finite set of states, $T \subseteq S\times S$ is a transition relation, $\ap$ is a finite set of atomic propositions, $L\colon S \to 2^{\ap}$ is a labeling function and $w\colon T\to \nat$ is a weight function. We assume that for every $s\in S$ exists $s'\in S$ such that $(s,s')\in T$. 
An initialized transition system is a TS $\dts=(S,T,\ap,L,w)$ with a distinctive initial state $s_{init}\in S$.
\end{definition}


A run of a TS $\dts$ is an infinite sequence $\rho = s_{0} s_{1}\ldots \in S^{\omega}$ such that for every $i\geq 0$ it holds $(s_{i},s_{i+1})\in T$. We use $\runinf(\rho)$ to denote the set of all states visited infinitely many times in the run $\rho$ and $\run^{\dts}(s)$ for the set of all runs of $\dts$ that start in $s\in S$. Let $\run^{\dts}=\bigcup_{s\in S}\run^{\dts}(s)$. A finite run $\sigma=s_0\ldots s_n$ of $\dts$ is a finite prefix of a run of $\dts$ and $\runfin^{\dts}(s)$ denotes the set of all finite runs of $\dts$ that start in $s\in S$. Let $\runfin^{\dts}=\bigcup_{s\in S}\runfin^{\dts}(s)$. The length $|\sigma|$, or number of stages, of a finite run $\sigma=s_0\ldots s_n$ is $n+1$ and $last(\sigma)=s_n$ denotes the last state of $\sigma$. With slight abuse of notation, we use $w(\sigma)$ to denote the weight of a finite run $\sigma=s_{0}\ldots s_{n}$, \ie $w(\sigma)=\sum_{i=0}^{n-1}w((s_{i},s_{i+1}))$. Moreover, $w^*(s,s')$ denotes the minimum weight of a finite run from $s$ to $s'$. Specifically, $w^*(s,s)=0$ for every $s\in S$ and if there does not exist a run from $s$ to $s'$, then $w^*(s,s')=\infty$. For a set $S'\subseteq S$ we let $w^*(s,S')=\min \limits_{s'\in S'}w^*(s,s')$. 
We say that a state $s'$ and a set $S'$ is reachable from $s$, iff $w^*(s,s')\neq \infty$ and $w^*(s,S')\neq \infty$, respectively. 

Every run $\rho=s_0s_1\ldots \in \run^{\dts}$, resp. $\sigma=s_0\ldots s_n\in \runfin^{\dts}$, induces a word $z=L(s_0)L(s_1)\ldots\in (2^{\ap})^{\omega}$, resp. $z=L(s_0)\ldots L(s_n)\in (2^{\ap})^+$, over the power set of $\ap$. 

A cycle of the TS $\dts$ is a finite run $\cyc=c_0\ldots c_m$ of $\dts$ for which it holds that $(c_m,c_0)\in T$. 



\begin{definition}
A sub-system of a $\dts=(S,T,\ap,L,w)$ is a TS $\scc = (S_{\scc}, T_{\scc},\ap, L|_{\scc},w|_{\scc})$, where $S_{\scc}\subseteq S$ and $T_{\scc}\subseteq T\cap (S_{\scc}\times S_{\scc})$. We use $L|_{\scc}$ to denote the labeling function $L$ restricted to the set $S_{\scc}$. Similarly, we use $w|_{\scc}$ with the obvious meaning. If the context is clear, we use $L,w$ instead of $L|_{\scc},w|_{\scc}$. 
A sub-system $\scc$ of $\dts$ is called strongly connected if for every pair of states $s,s'\in S_{\scc}$, there exists a finite run $\sigma \in \runfin^{\scc}(s)$ such that $last(\sigma)=s'$. 
A strongly connected component (SCC) of $\dts$ is a maximal strongly connected sub-system of $\dts$.
We use $\SCC(\dts)$ to denote the set of all strongly connected components of $\dts$.  
\end{definition}

Strongly connected components of a TS $\dts$ are pairwise disjoint. Hence, the cardinality of the set $\SCC(\dts)$ is bounded by the number of states of $\dts$ and the set can be computed using Tarjan's algorithm~\cite{tarjan}.



\begin{definition}
Let $\dts=(S,T,\ap,L,w)$ be a TS. A control strategy for $\dts$ is a function $\strat \colon \runfin^{\dts}\to S$ such that for every $\sigma\in \runfin^{\dts}$, it holds that $(last(\sigma),\strat(\sigma))\in T$.
\end{definition}

A strategy $\strat$ for which $\strat(\sigma_1)=\strat(\sigma_2)$, for all finite runs $\sigma_1,\sigma_2\in \runfin^{\dts}$ with $last(\sigma_1)=last(\sigma_2)$, is called memoryless. In that case, $\strat$ is a function $\strat \colon S\to S$.

A strategy is called finite-memory if it is defined as a tuple $\strat = (M,\mathsf{next},\Delta,\mathsf{start})$, where $M$ is a finite set of modes, $\Delta\colon M\times S\to M$ is a transition function, $\mathsf{next}\colon M\times S\to S$ selects a state of $\dts$ to be visited next, and $\mathsf{start}\colon S\to M$ selects the starting mode for every $s\in S$.

A run induced by a strategy $\strat$ for $\dts$ is a run $\rho_{\strat}=s_{0}s_{1}\ldots \in \run^{\dts}$ for which $s_{i+1}=\strat(\rho_{\strat}^{(i)})$ for every $i\geq 0$. For every $s\in S$, there is exactly one run induced by $\strat$ that starts in $s$. A finite run induced by $\strat$ is $\sigma_{\strat}\in \runfin^{\dts}$, which is a finite prefix of a run $\rho_{\strat}$ induced by $\strat$.

Let $\strat$ be a strategy, finite-memory or not, for a TS $\dts$. For every state $s\in S$, the run $\rho_{\strat}\in \run^{\dts}(s)$ induced by $\strat$ satisfies $\runinf(\rho_{\strat})\subseteq S_{\scc}$ for some $\scc \in \SCC(\dts)$~\cite{baierBook}. We say that $\strat$ leads $\dts$ from the state $s$ to the SCC $\scc$.

%



\begin{definition}
Linear Temporal Logic (LTL) formulas over the set $\ap$ are formed according to the following grammar:
\begin{equation*}
\phi::= true\mid a \mid \neg \phi \mid \phi \vee \phi \mid \phi \wedge \phi \mid \X \, \phi \mid \phi \, \U \, \phi \mid \G\, \phi \mid \F\, \phi,
\end{equation*}
where $a\in \ap$ is an atomic proposition, $\neg$, $\vee$ and $\wedge$ are standard Boolean connectives, and $\X$ (\emph{next}), $\U$ (until), $\G$ (always) and $\F$ (eventually) are temporal operators.
\end{definition}

The semantics of LTL is defined over words over~$2^{\ap}$, such as those generated by the runs of a TS $\dts$ (for details see \eg \cite{baierBook}). For example, a word $w \in {(2^{\ap})}^\omega$ satisfies $\G \,\phi$ and $\F \, \phi$ if $\phi$ holds in $w$ always and eventually, respectively. 
If the word induced by a run of $\dts$ satisfies a formula $\phi$, we say that the run satisfies $\phi$. 
We call $\phi$ satisfiable in $\dts$ from $s\in S$ if there exists a run $\rho \in \run^{\dts}(s)$ that satisfies $\phi$. 

Having an initialized TS $\dts$ and an LTL formula $\phi$ over $\ap$, the formal synthesis problem aims to find a strategy $\strat$ for $\dts$ such that the run $\rho_{\strat}\in \run^{\dts}(s_{init})$ induced by $\strat$ 
satisfies $\phi$. In that case we also say that the strategy $\strat$ satisfies $\phi$. The formal synthesis problem can be solved using principles from model checking methods~\cite{baierBook}. 
Specifically, $\phi$ is translated to a B\"{u}chi automaton and the system combining the B\"{u}chi automaton and the TS $\dts$ is analyzed. 

\begin{definition}
A B\"{u}chi automaton (BA) is a tuple $\ba=(Q,2^{\ap},\delta,q_{0},F)$, where $Q$ is a non-empty finite set of states, $2^{\ap}$ is the alphabet, $\delta \subseteq Q\times 2^{\ap} \times Q$ is a transition relation such that for every $q\in Q$, $a\in 2^{\ap}$, there exists $q'\in Q$ such that $(q,a,q')\in \delta$, $q_{0}\in Q$ is the initial state, and $F\subseteq Q$ is a set of accepting states.
\end{definition}

A run $q_{0}q_{1}\ldots Q^{\omega}$ of $\ba$ is an infinite sequence such that for every $i\geq 0$ there exists $a_i\in 2^{\ap}$ with $(q_i,a_i,q_{i+1})\in \delta$. The word $a_{0}a_{1}\ldots \in (2^{\ap})^{\omega}$ is called the word induced by the run $q_{0}q_{1}\ldots $. A run $q_{0}q_{1}\ldots $ of $\ba$ is accepting if there exist infinitely many $i\geq 0$ such that $q_i$ is an accepting state.

For every LTL formula $\phi$ over $\ap$, one can construct a B\"{u}chi automaton $\ba_{\phi}$ such that the accepting runs are all and only words over $2^{\ap}$ satisfying $\phi$ \cite{ltlBaProof}. We refer readers to \cite{fastLtlToBa,efficientBaFromLtl} for algorithms and to online implementations such as \cite{ltl2ba}, to translate an LTL formula to a BA.

\begin{definition}
Let $\dts=(S,T,\ap,L,w)$ be an initialized TS and $\ba=(Q,2^{\ap},\delta,q_{0},F)$ be a B\"{u}chi automaton. The product $\product$ of $\dts$ and $\ba$ is a tuple
$\product = (S_{\product}, T_{\product}, s_{\product init},\ap, L_{\product},F_{\product},w_{\product}),$
where $S_{\product}=S\times Q$, $T_{\product}\subseteq S_{\product}\times S_{\product}$ is a transition relation such that for every $(s,q),(s',q')\in S_{\product}$ it holds that $((s,q),(s',q'))\in T_{\product}$ if and only if $(s,s')\in T$ and $(q,L(s),q')\in \delta$, $s_{\product init}=(s_{init},q_0)$ is the initial state, $L_{\product}((s,q))=L(s)$ is a labeling function, $F_{\product}=S\times F$ is a set of accepting states, and $w_{\product}(((s,q),(s',q')))=w((s,s'))$ is a weight function.
\end{definition}

The product $\product$ can be viewed as an initialized TS 
with a 
set of accepting states. Therefore, we adopt the definitions of a run $\rho$, a finite run $\sigma$, its 
weight $w_{\product}(\sigma)$, and sets 
$\run^{\product}((s,q))$, $\run^{\product}$, $\runfin^{\product}((s,q))$ and $\runfin^{\product}$ from above. Similarly, a cycle $\cyc$ of $\product$, a strategy $\strat_{\product}$ for $\product$ and runs $\rho_{\strat_{\product}},\sigma_{\strat_{\product}}$ induced by 
$\strat_{\product}$ are defined in the same way as for a TS. On the other hand, $\product$ can be viewed as a weighted BA over the trivial alphabet with a labeling function, which gives us the definition of an accepting run of $\product$.

Using the projection on the first component, every run $(s_0,q_0)(s_1,q_1)\ldots $ and finite run $(s_0,q_0)\ldots (s_n,q_n)$ of $\product$ corresponds to a run $s_0s_1\ldots$ and a finite run $s_0\ldots s_n$ of $\dts$, respectively. Vice versa, for every run $s_0s_1\ldots$ and finite run $s_0\ldots s_n$ of $\dts$, there exists a run $(s_0,q_0)(s_1,q_1)\ldots $ and finite run $(s_0,q_0)\ldots (s_n,q_n)$. Similarly, every strategy for $\product$ projects to a strategy for $\dts$ and for every strategy for $\dts$ there exists a strategy for $\product$ that projects to it. The projection of a finite-memory strategy for $\product$ is also finite-memory.

Since $\product$ can be viewed as a TS, we also adopt the definitions of a sub-system and a strongly connected component.

\begin{definition}
Let $\product = (S_{\product}, T_{\product}, s_{\product init},\ap, L_{\product},F_{\product},w_{\product})$ be the product of an initialized TS $\dts$ and a BA $\ba$. An accepting strongly connected component (ASCC) of $\product$ is 
an SCC  $\scc = (S_{\scc}, T_{\scc},\ap, L_{\product},w_{\product})$ such that the set $S_{\scc} \cap F_{\product}$ is non-empty and we refer to it as the set $F_{\scc}$ of accepting states of $\scc$. 
We use $\ASCC(\product)$ to denote the set of all ASCCs of $\product$ that are reachable from the initial state $s_{\product init}$. 
\end{definition}



In our work, we always assume that $\ASCC(\product)$ is non-empty, \ie the given LTL formula is satisfiable in the TS.


\section{Problem Formulation}\label{sec:pf}

Consider an initialized weighted transition system $\dts=(S,T,\ap,L,w)$. The weight $w((s,s'))$ of a transition $(s,s')\in T$ represents the amount of time that the transition takes and the system starts at time $0$. We use $t_{n}$ to denote the point in time after the $n$-th transition of a run, \ie initially the system is in a state $s_{0}$ at time $t_{0}=0$ and after a finite run $\sigma\in \runfin^{\dts}(s_{0})$ of length $n+1$ the time is $t_{n}=w(\sigma)$.

We assume there is a dynamic \emph{penalty} associated with every state $s\in S$. In this paper, we address the following model of penalties. Nevertheless, as we discuss in Sec.\ref{sec:discussion}, the algorithms presented in the next section provide optimal solution for a much broader class of penalty dynamics. The penalty is a rational number between $0$ and $1$ that is increasing every time unit by $\frac{1}{r}$, where $r\in \nat$ is a given rate. Always when the penalty is $1$, in the next time unit the penalty remains $1$ or it drops to $0$ according to a given probability distribution. Upon the visit of a state, the corresponding penalty is incurred. We assume that the visit of the state does not affect the penalty's value or dynamics. Formally, the penalties are defined by a \emph{rate} $r\in \nat$ and a \emph{penalty probability function} $p\colon S\to (0,1]$, where $p(s)$ is the probability that if the penalty in a state $s$ is $1$ then in the next time unit the penalty remains $1$, and $1-p(s)$ is the probability of the penalty dropping to $0$. The penalties are described using a function $\pen\colon S\times \nat_{0}\to \{\tfrac{i}{r}\mid i\in \{0,1,\ldots ,r\}\}$, such that $\pen(s,t)$ is the penalty in a state $s\in S$ at time $t\in \nat_{0}$. For $s\in S$, $\pen(s,0)$ is a uniformly distributed random variable with values in the set $\{\tfrac{i}{r}\mid i\in \{0,1,\ldots ,r\}\}$ and for $t\geq 1$
\begin{equation}\footnotesize \label{eq:pen}
\pen(s,t)=\begin{cases}
\pen(s,t-1)+\frac{1}{r} & \text{if }\pen(s,t-1)<1,\\
x & \text{otherwise},
\end{cases}
\end{equation}
where $x$ is a random variable such that $x=1$ with probability $p(s)$ and $x=0$ otherwise. We use 
\begin{equation} \label{eq:penexp}
\penexp(s)=(1-p(s))\cdot \tfrac{1}{2}+p(s)\cdot 1=\tfrac{1}{2}(1+p(s))
\end{equation}
to denote the expected value of the penalty in a state $s\in S$. Please note that $\tfrac{1}{2}\leq
\penexp(s)\leq 1$, for every $s\in S$.

In our setting, the penalties are sensed only locally in the states in close proximity from the current state. To be specific, we assume a {\it visibility range} $\vis \in \nat$ is given.  If the
system is  in a state $s\in S$ at time $t$,  the penalty $g(s',t)$ of a state $s'\in S$ is
observable if and only if $s'\in \Vis(s)=\{s'\in S\mid w^*(s,s')\leq \vis\}$. The set $\Vis(s)$ is also called the set of states visible from $s$.

The problem we consider in this paper combines the formal synthesis problem with long-term optimization of the expected amount of penalties
incurred during the system's execution. We assume that 
the specification is given as an LTL formula $\phi$ of the form 
\begin{equation}\label{eq:ltlsur}
\phi=\varphi\, \wedge \, \G\F\pisur,
\end{equation}
where $\varphi$ is an LTL formula over $\ap$ and $\pisur\in \ap$. 
This formula requires that the system satisfies $\varphi$ and surveys the states satisfying the property $\pisur$ infinitely often. 
We say that every visit of a state from the set $S_{\sur}=\{s\in S\mid \pisur \in L(s)\}$ completes a surveillance cycle. Specifically, starting from the initial state, the first visit of $S_{\sur}$ after the initial state completes the first surveillance cycle of a run. Note that a surveillance cycle is not a cycle in the sense of the definition of a cycle of a TS in Sec.~\ref{sec:prelim}. For a finite run $\sigma$ such that $last(\sigma)\in S_{\sur}$, $\cycles(\sigma)$ denotes the number of complete surveillance cycles in $\sigma$, otherwise $\cycles(\sigma)$ is the number of complete surveillance cycles plus one. 
We define a function $V_{\dts,\strat}\colon S\to \mathbb{R}^+_0$ such that $V_{\dts,\strat}(s)$ is the expected average cumulative penalty per surveillance cycle (APPC) incurred under a strategy $\strat$ for $\dts$ starting from a state $s\in S$:
\begin{equation}\label{eq:appc}
V_{\dts,\strat}(s)=\limsup_{n\to \infty} E\Big(\frac{\sum_{i=0}^{n}g(\rho_{\strat}(i),w(\rho_{\strat}^{(i)}))}{\cycles(\rho_{\strat}^{(n)})}\Big),
\end{equation}
where $\rho_{\strat}\in \run^{\dts}(s)$ is the run induced by $\strat$ starting from $s$ and $E(\cdot)$ denotes the expected value. In this paper, we consider the following problem:

\medskip

\begin{problem}\label{pf:appc}
Let $\dts=(S,T,AP,L,w)$ be an initialized TS, with penalties defined by a rate $r\in \nat$ and penalty probabilities $p\colon S\to (0,1]$.
Let $\vis \in \nat$ be a visibility range 
and $\phi$ an LTL formula over the set $\ap$ of the form in Eq.~(\ref{eq:ltlsur}). Find a strategy $\strat$ for $\dts$ such that $\strat$ satisfies $\phi$ and among all strategies satisfying $\phi$, $\strat$ minimizes the APPC value $V_{\dts,\strat}(s_{init})$ defined in Eq.~(\ref{eq:appc}).
\end{problem}

\medskip

In the next section, we propose two algorithms solving the above problem. The first algorithm operates offline, without the deployment of the system, and therefore, without taking advantage of the local sensing of penalties. On the other hand, the second algorithm computes the strategy in real-time 
by locally improving the offline strategy according to the penalties observed from the current state and their simulation over the next $h$ time units, where $h\geq 1$ is a natural number, a user-defined \emph{planning horizon}.


\section{Solution}\label{sec:solution}


The two algorithms work with the product $\product = (S_{\product}, T_{\product}, s_{\product init},\ap, L_{\product},F_{\product},w_{\product})$ of the initialized TS $\dts$ and a B\"{u}chi automaton $\ba_{\phi}$ for the LTL formula $\phi$. To project the penalties from $\dts$ to $\product$, we define the penalty in a state $(s,q)\in S_{\product}$ at time $t$ as $\pen((s,q),t)=\pen(s,t)$. 
We also adopt the visibility range $v$ and the set $\Vis((s,q))$ of all states visible from $(s,q)$ is defined as for a state of $\dts$. The APPC function $V_{\product,\strat_{\product}}$ of a strategy $\strat_{\product}$ for $\product$ is then defined according to Eq.~(\ref{eq:appc}). We use the correspondence between the strategies for $\product$ and $\dts$ to find a strategy for $\dts$ that solves Problem~\ref{pf:appc}. Let $\stratoff$ be a strategy for $\product$ such that the run induced by $\stratoff$ visits the set $F_{\product}$ infinitely many times and at the same time, the APPC value $V_{\product,\stratoff}(s_{\product init})$ is minimal among all strategies that visit $F_{\product}$ infinitely many times. It is easy to see that $\stratoff$ projects to a strategy $\strat$ for $\dts$ that solves Problem~\ref{pf:appc} and $V_{\dts,\strat}(s_{init})=V_{\product,\strat_{\product}}(s_{\product init})$.

The offline algorithm leverages ideas from formal methods. Using the automata-based approach to model checking, one can construct a strategy $\stratoff^{\phi}$ for $\product$ that visits at least one of the accepting states infinitely many times. On the other hand, using graph theory, we can design a strategy $\stratoff^V$ that achieves the minimum APPC value among all strategies of $\product$ that do not cause an immediate, unrepairable violation of $\phi$, \ie $\phi$ is satisfiable from every state of the run induced by $\stratoff^V$. However, we would like to have a strategy $\stratoff$ satisfying both properties at the same time.  
To achieve that, we employ a technique from game theory presented in~\cite{chatterjeeMFCS11}. Intuitively, 
we combine two strategies $\stratoff^{\phi}$ and $\stratoff^V$ to create a new strategy $\stratoff$.
The strategy  $\stratoff$ is played in rounds, where each round consists of two phases. In the first phase, we play the strategy $\stratoff^{\phi}$ 
until an accepting state is reached. We say that the system is to achieve the mission subgoal. The second phase applies the strategy $\stratoff^V$. 
The aim is to maintain the expected average cumulative penalty per surveillance cycle in the current round, and we refer to it as the average subgoal. The number of steps for which we apply $\stratoff^V$ is computed individually every time we enter the second phase of a round.  

The online algorithm constructs a strategy $\straton$ by locally improving the strategy $\stratoff$
computed by the offline algorithm.
Intuitively, 
we compare applying $\stratoff$ for several steps to reach a specific state or set of states of
$\product$, to executing different local paths to reach the same state or set. We consider a finite
set of finite runs leading to the state, or set, containing the finite run induced by
$\stratoff$, choose the one that is expected to minimize the average cumulative penalty per surveillance cycle incurred in the current round and apply the first transition of the chosen run. 
The process continues until the state, or set, is reached, and then it starts over again.


\subsection{Probability measure}\label{subsec:arpsprob}

Let $\strat_{\product}$ be a strategy for $\product$ and $(s,q)\in S_{\product}$ a state of $\product$. For a finite run $\sigma_{\strat_{\product}} \in \runfin^{\product}((s,q))$ induced by the strategy $\strat_{\product}$ starting from the state $(s,q)$ and a sequence $\tau \in (\{\tfrac{i}{r}\mid 0\leq i\leq r\})^{+}$ of length $|\sigma_{\strat_{\product}}|$, we
call $(\sigma_{\strat_{\product}},\tau)$
a finite pair.
Analogously, an infinite pair $(\rho_{\strat_{\product}},\kappa)$ 
consists of the run $\rho_{\strat_{\product}}\in \run^{\product}((s,q))$ induced by the strategy $\strat_{\product}$ 
and an infinite sequence $\kappa \in (\{\tfrac{i}{r}\mid 0\leq i\leq r\})^{\omega}$.
A cylinder set $Cyl((\sigma_{\strat_{\product}},\tau))$ of a finite pair $(\sigma_{\strat_{\product}},\tau)$ is the set of all 
infinite pairs $(\rho_{\strat_{\product}},\kappa)$ such that 
$\tau$ is a prefix of $\kappa$. 

Consider the $\sigma$-algebra generated by the set of cylinder sets of all 
finite pairs $(\sigma_{\strat_{\product}},\tau)$, where $\sigma_{\strat_{\product}} \in \runfin^{\product}((s,q))$ is a finite run induced by the strategy $\strat_{\product}$ starting from the state $(s,q)$ and $\tau \in (\{\tfrac{i}{r}\mid 0\leq i\leq r\})^{+}$ is of length $|\sigma_{\strat_{\product}}|$. 
From classical concepts in probability theory~\cite{uniqueprobmeasure}, there exists a unique probability measure $\pr^{\product,\strat_{\product}}_{(s,q)}$ on the $\sigma$-algebra such that for a finite pair $(\sigma_{\strat_{\product}},\tau)$
$$\pr^{\product,\strat_{\product}}_{(s,q)}(Cyl((\sigma_{\strat_{\product}},\tau)))$$
is the probability that the penalties incurred in the first $|\sigma_{\strat_{\product}}|+1$ stages when applying the strategy $\strat_{\product}$ in $\product$ from the state $(s,q)$ are given by the sequence $\tau$, \ie 
$$\pen(\sigma_{\strat_{\product}}(i),w_{\product}(\sigma_{\strat_{\product}}^{(i)}))=\tau(i)$$
for every $0\leq i\leq |\sigma_{\strat_{\product}}|$. This probability is given by the penalty dynamics and therefore, can be computed from the rate $r$ and the penalty probability function $p$. For a set $X$ of infinite pairs, an element of the above $\sigma$-algebra, the probability $\pr^{\product,\strat_{\product}}_{(s,q)}(X)$ is the probability that under $\strat_{\product}$ starting from $(s,q)$ the infinite sequence of penalties received in the visited states is $\kappa$ where $(\rho_{\strat_{\product}}, \kappa)\in X$.


\subsection{Offline control}\label{subsec:appcoff}

In this section, we construct a strategy $\stratoff$ for $\product$ that projects to a strategy $\strat$ for $\dts$ solving Problem~\ref{pf:appc}. 
The strategy $\stratoff$ has to visit $F_{\product}$ infinitely many times and therefore, $\strat_{\product}$ must lead from $s_{\product init}$ to an ASCC. For an 
$\scc \in \ASCC(\product)$, we denote $V_{\scc}^*((s,q))$ the minimum expected average cumulative penalty per surveillance cycle that can be achieved in $\scc$ starting from $(s,q)\in S_{\scc}$. Since $\scc$ is strongly connected, this value is the same for all the states in $S_{\scc}$ and is referred to as $V_{\scc}^*$. It is associated with a cycle $\cyc_{\scc}^V=c_0\ldots c_m$ of $\scc$ witnessing the value, \ie 
$$\tfrac{1}{|\cyc_{\scc}^V\cap S_{\scc \sur}|}\sum \limits_{i=0}^{m}\penexp(c_i)=V_{\scc}^*$$
where $S_{\scc \sur}$ is the set of all states of $\scc$ labeled with $\pisur$.
Since $\scc$ is an ASCC, it holds $S_{\scc \sur}\neq \emptyset$. 

We design an algorithm that finds the value $V_{\scc}^*$ and a cycle $\cyc^V_{\scc}$ for an
ASCC $\scc$.
The algorithm first reduces $\scc$ to a TS $\overline{\scc}$ and then applies the Karp's
algorithm~\cite{karpmaxmean} on  $\overline{\scc}$. The Karp's algorithm finds for a directed graph
with values on edges  a cycle with minimum value per edge also called the minimum mean cycle.  The value $V_{\scc}^*$ and cycle $\cyc^V_{\scc}$ are synthesized from the minimum mean cycle.


Let $\scc=(S_{\scc}, T_{\scc},\ap, L_{\product},w_{\product})$ be an 
ASCC of $\product$. For simplicity, we use singletons such as $u,u_i$ to denote the states of $\product$ in this paragraph. We construct a TS
\begin{equation*}
\overline{\scc}=(S_{\scc \sur}, T_{\overline{\scc}},\ap,L_{\product},w_{\overline{\scc}}),
\end{equation*}
and a function $\seq \colon T_{\overline{\scc}}\to \runfin^{\scc}$ for which it holds that
$(u,u')\in T_{\overline{\scc}}$ if and only if there exists a finite run in $\scc$ from $u\in
S_{\scc \sur}$ to $u'\in S_{\scc \sur}$ with one surveillance cycle, \ie between $u$ and $u'$ no
state labeled with $\pisur$ is visited. Moreover,  the run $\seq((u,u'))=u_0\ldots u_n$ is such that
$u=u_0$ and $\sigma=u_0\ldots u_n u'$ is the finite run in $\scc$ from $u$ to $u'$ with one
surveillance cycle that minimizes the expected sum of penalties received during $\sigma$ among all
finite runs in $\scc$ from $u$ to $u'$ with one surveillance cycle. The TS can be constructed from
$\scc$ by an iterative algorithm eliminating the states from $S_{\scc}\backslash S_{\scc \sur}$ one
by one, in arbitrary order. At the beginning let  $\overline{\scc}=\scc$, $ T_{\overline{\scc}} =
T_{\scc}$,  and for every transition
$(u,u')\in T_{\overline{\scc}}$ let  $\seq((u,u'))=u$. 
The procedure for eliminating $u\in S_{\scc}\backslash S_{\scc \sur}$ proceeds as follows. Consider every $u_1\neq u,u_2\neq u$ such that
$(u_1,u),(u,u_2)\in T_{\overline{\scc}}$. If the transition $(u_1,u_2)$ is not in $T_{\overline{\scc}}$, add $(u_1,u_2)$ to $T_{\overline{\scc}}$ and define 
$\seq((u_1,u_2))=\seq((u_1,u)).\seq((u,u_2))$,
where $.$ is the concatenation of sequences. If $T_{\overline{\scc}}$ already contains the transition $(u_1,u_2)$ and $\seq((u_1,u_2))=\sigma$, we set
$\seq((u_1,u_2))=\seq((u_1,u)).\seq((u,u_2))$,
if
\begin{equation*}\small
\begin{split}
\sum \penexp\big(\seq((u_1,u)).\seq((u,u_2))\big)\leq \sum \penexp\big(\sigma\big),
\end{split}
\end{equation*}
where $\sum \penexp(x)$ for a run $x$ is the sum of $\penexp(x(i))$ for every state $x(i)$ of $x$, otherwise we let $\seq((u_1,u_2))=\sigma$. The weight $w_{\overline{\scc}}((u_1,u_2))=\sum \penexp(\seq((u_1,u_2)))$.
Once all pairs $u_1,u_2$ are handled,  remove $u$ from $S_{\overline{\scc}}$ and all adjacent transitions from $T_{\overline{\scc}}$.  Fig.~\ref{fig:red} demonstrates one iteration of the algorithm. 

Consequently, we apply the Karp's algorithm on the oriented graph with vertices $S_{\scc \sur}$,
edges $T_{\overline{\scc}}$ and values on edges $w_{\overline{\scc}}$. Let
$\cyc_{\overline{\scc}}=u_0\ldots u_m$ be the  minimum mean cycle of this  graph. Then it holds
\begin{equation*}\small
\begin{split}
V_{\scc}^* & =\tfrac{1}{m+1}\sum \limits_{i=0}^{m}\penexp\big(\seq((u_i,u_{i+1 \intmod(m+1)})\big),\\
\cyc_{\scc}^V & = \seq((u_0,u_1)).\, \ldots \,. \seq((u_{m-1},u_m)).\seq((u_m,u_0)).
\end{split}
\end{equation*}
\begin{figure}[t]
\begin{center}
\scalebox{0.45}{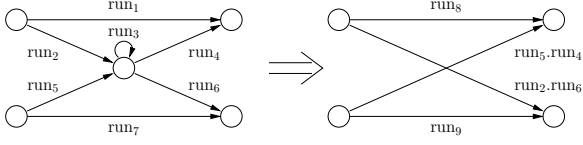}
\end{center}
\caption{An example of elimination of a state during the reduction of an ASCC $\scc$. The finite run $\seq_8$ is equal to the one of the finite runs $\seq_1$ and $\seq_2.\seq_4$ that minimizes the sum of expected penalties in the states of the run. Similarly, $\seq_9$ is one of the finite runs $\seq_7$ and $\seq_5.\seq_6$.}
\label{fig:red}
\end{figure}
When the APPC value and the corresponding cycle is computed for every 
ASCC of $\product$, we choose the ASCC that minimizes the APPC value. We denote this ASCC $\scc=(S_{\scc}, T_{\scc},\ap, L_{\product},w_{\product})$ and $\cyc_{\scc}^V=c_0\ldots c_m$.

The mission subgoal aims to reach an accepting state from the set $F_{\scc}$. The corresponding strategy $\stratoff^{\phi}$ is such that from every state $(s,q)\in S_{\product}\backslash F_{\scc}$ that can reach the set $F_{\scc}$, we follow one of the finite runs with minimum weight from $(s,q)$ to $F_{\scc}$. That means, $\stratoff^{\phi}$ is a memoryless strategy such that for $(s,q)\in S_{\product}\backslash F_{\scc}$ with $w_{\product}^*((s,q),F_{\scc})<\infty$ it holds $\stratoff^{\phi}((s,q))=(s',q')$ where 
$$w_{\product}((s,q),(s',q'))=w_{\product}^*((s,q),F_{\scc})-w_{\product}^*((s',q'),F_{\scc}).$$

The strategy $\stratoff^V$ for the average subgoal is given by the cycle $\cyc_{\scc}^V=c_0\ldots c_m$ of the ASCC $\scc$. Similarly to the mission subgoal, for a state $(s,q)\in S_{\product}\backslash \cyc_{\scc}^V$ with $w^*_{\product}((s,q),\cyc_\scc^V)<\infty$, the strategy $\stratoff^V$ follows one of the finite runs with minimum weight to $\cyc_{\scc}^V$. For a state $c_i\in \cyc_{\scc}^V$, it holds $\stratoff^V(c_i)=c_{i+1 \intmod(m+1)}$. If all the states of the cycle $\cyc_{\scc}^V$ are distinct, the strategy $\stratoff^V$ is memoryless, otherwise it is finite-memory.
%


\begin{proposition}\label{th:limprob}
For the strategy $\stratoff^V$ and every state $(s,q)\in S_{\scc}$, it holds
\begin{equation*}
\lim \limits_{n \to \infty} \pr^{\scc,\stratoff^V}_{(s,q)}\big(\frac{\sum_{i=0}^{n}g(\rho_{\stratoff^V}(i),w_{\product}(\rho_{\stratoff^V}^{(i)}))}{\cycles(\rho_{\stratoff^V}^{(n)})}\leq V_{\scc}^*\big)=1.
\end{equation*}
Equivalently, for every state $(s,q)\in S_{\scc}$ and every $\epsilon>0$, there exists $j(\epsilon)\in \nat$ such that if the strategy $\stratoff^V$ is played from the state $(s,q)$ until at least $l\geq j(\epsilon)$ surveillance cycles are completed, then the average cumulative penalty per surveillance cycle incurred in the performed finite run is at most $V_{\scc}^*+\epsilon$ with probability at least $1-\epsilon$.
\end{proposition}

\begin{proof}
(Sketch.) The proof is based on the fact that the product $\product$ with dynamic penalties can be translated into a Markov decision process (MDP)~(see \eg \cite{mdps}) with static penalties. The run $\rho_{\stratoff^V}$ corresponds to a Markov chain (see \eg \cite{markovchains}) of the MDP. Moreover, the cycle $\cyc_{\scc}^*$ corresponds to the minimum mean cycle of the reduced TS $\overline{\scc}$. Hence, the equation in the theorem is equivalent to the property of MDPs with static penalties proved in~\cite{chatterjeeMFCS11} regarding the minimum expected penalty incurred per stage.  
\end{proof}

\begin{remark}

Assume there exists a state $(s,q)\in S_{\product}$ with $p((s,q))=0$, \ie if the penalty in $(s,q)$ is $1$, it always drops to $0$. The dynamics of the penalty in $(s,q)$ is not probabilistic and if we visit $(s,q)$ infinitely many times, the expected average penalty incurred in $(s,q)$ might differ from $\penexp((s,q))$. That can cause violation of Prop.~\ref{th:limprob}.

\end{remark}

Now we describe the strategy $\stratoff$. 
It is played in rounds, where each round consists of two phases, one for each subgoal. The first round starts at the beginning of the execution of the system in the initial state $s_{\product init}$ of $\product$. Let $i$ be the current round. In the first phase of the round the strategy $\stratoff^{\phi}$ is applied until an accepting state of the ASCC $\scc$ is reached. We use $k_i$ to denote the number of steps we played the strategy $\stratoff^{\phi}$ in round $i$. Once the mission subgoal is fulfilled, the average subgoal becomes the current subgoal. In this phase, we play the strategy $\stratoff^V$ until the number of completed surveillance cycles in the second phase of the current round is $l_i\geq \max \{j(\tfrac{1}{i}),i\cdot k_i\}$.

%

\begin{theorem}\label{th:offcorr}
The strategy $\stratoff$ projects to a strategy $\strat$ of $\dts$ that solves Problem~\ref{pf:appc}.
\end{theorem}

\begin{proof}
From the fact that the ASCC $\scc$ is reachable from the initial state $s_{\product init}$ and from the construction of $\stratoff^{\phi}$, it follows that $\scc$ is reached from $s_{\product init}$ in finite time. In every round $i$ of the strategy $\stratoff$, an accepting state is visited. Moreover, from Prop.~\ref{th:limprob} and the fact that $l_i\geq \max \{j(\tfrac{1}{i}),i\cdot k_i\}$, it can be shown that the average cumulative penalty per surveillance cycle incurred in the $i$-th round is at most $V_{\scc}^*+\tfrac{2}{i}$ with probability at least $1-\tfrac{1}{i}$. Therefore, in the limit, the run induced by $\stratoff$ satisfies the LTL specification and reaches the optimal average cumulative penalty per surveillance cycle $V_{\scc}^*$ with probability $1$.
\end{proof}

%
%
Note that, in general, the strategy $\stratoff$ is not finite-memory. The reason is that in the modes of the finite-memory strategy we would need to store the number of steps spent so far in the first phase $k_i$ and the number $l_i$ of the surveillance cycles in the second phase of a given round. Since $j(\tfrac{1}{i})$ is generally increasing with $i$, we would need infinitely many modes to be able to count the number of surveillance cycles in the second phase. However, if there exists a cycle $\cyc_{\scc}^V$ of the SCC $\scc$ corresponding to $V_{\scc}^*$ that contains an accepting state, then the finite-memory strategy $\stratoff^V$ for the average subgoal maps to a strategy of $\dts$ solving Problem~\ref{pf:appc}, which is therefore in the worst case finite-memory as well.


\subsubsection*{Complexity}

The size of a BA for an LTL formula $\phi$ is in the worst case $2^{\mathcal{O}(|\phi|)}$, where $|\phi|$ is the size of $\phi$~\cite{fastLtlToBa}. However, the actual size of the BA is in practice often quite small. The size of the product $\product$ is $\mathcal{O}(|S|\cdot 2^{\mathcal{O}(|\phi|)})$. To compute the minimum weights $w^*((s,q),(s',q'))$ between every two states of $\product$ we use Floyd-Warshall algorithm taking $\mathcal{O}(|S_{\product}|^3)$ time. Tarjan's algorithm~\cite{tarjan} is used to compute the set $\SCC(\product)$ in time $\mathcal{O}(|S_{\product}|+|T_{\product}|)$. The reduction of an 
ASCC $\scc$ can be computed in time $\mathcal{O}(|S_{\scc}|\cdot |T_{\scc}|^2)$. The Karp's algorithm~\cite{karpmaxmean} finds the optimal APPC value and corresponding cycle in time $\mathcal{O}(|S_{\overline{\scc}}|\cdot |T_{\overline{\scc}}|)$. The main pitfall of the algorithm is to compute the number $j(\tfrac{1}{i})$ of surveillance cycles needed in the second phase of the current round $i$ according to Prop.~\ref{th:limprob}. Intuitively, we need to consider the finite run $\sigma_{\stratoff^V,k}$ induced by the strategy $\stratoff^V$ from the current state that contains $k=1$ surveillance cycles, and compute the sum of probabilities $\pr^{\product,\strat_{\product}}_{(s,q)}(Cyl((\sigma_{\stratoff^V,k},\tau)))$ for every $\tau$ with the average cumulative penalty per surveillance cycle less or equal to $V_{\scc}^*+\tfrac{1}{i}$. If the total probability is at least $1-\tfrac{1}{i}$, we set $j(\tfrac{1}{i})=k$, otherwise we increase $k$ and repeat the process. For every $k$, there exist $r^{|\sigma_{\stratoff^V,k}|}$ sequences $\tau$. To partially overcome this issue, we compute the number $j(\tfrac{1}{i})$ only at the point in time, when the number of surveillance cycles in the second phase of the current round $i$ is $i\cdot k_i$ and the average cumulative penalty in this round is still above $V_{\scc}^*+\tfrac{2}{i}$. As the simulation results in Sec.~\ref{sec:casestudy} show, this happens only rarely, if ever.

\subsection{Online control}\label{subsec:appcon}

The online algorithm locally improves the strategy $\stratoff$ 
according to the values of penalties observed from the current state and their simulation in the next $h$ time units. The resulting strategy $\straton$ is again played in rounds. However, in each step of the strategy $\straton$, we consider a finite set of finite runs starting from the current state, choose one according to an optimization function, and apply its first transition.

Throughout the rest of the section we use the following notation. We use singletons such as $u,u_i$ to denote the states of $\product$. Let $\sigma_{\mathrm{all}}\in \runfin^{\product}(s_{\product init})$ denote the finite run executed by $\product$ so far. Let $i$ be the current round of the strategy $\straton$ 
and $\sigma_i=u_{i,0}\ldots u_{i,k}$ the finite run executed so far in this round, \ie $u_{i,k}$ is the current state of $\product$. We use $t_{i,0},\ldots ,t_{i,k}$ to denote the points in time when the states $u_{i,0},\ldots ,u_{i,k}$ were visited, respectively. 

The optimization function $f\colon \runfin^{\product}(u_{i,k}) \to [0,1]$ assigns every finite run $\sigma=u_0\ldots u_n$ starting from the current state a value $f(\sigma)$ that is the expected average cumulative penalty per surveillance cycle that would be incurred in the round $i$, if the run $\sigma$ was to be executed next, \ie
\begin{equation}
\small
f(\sigma)=\frac{\sum \limits_{j=0}^{k}\pen(u_{i,j},t_{i,j}) + \sum \limits_{j=1}^{n}\pensim(u_j,t_{i,k}+w_{\product}(\sigma^{(j)}))}{\cycles(\sigma_i.\sigma(1)\ldots last(\sigma))},
\label{eq:optf}
\end{equation}
where $\pensim(u_j,t_{i,k}+w_{\product}(\sigma^{(j)}))$ is the simulated expected penalty incurred in the state $u_{j}$ at the time of its visit. If the visit occurs within the next $h$ time units and the state $u_j$ is visible from the current state $u_{i,k}$, we simulate the penalty currently observed in $u_j$ over $w_{\product}(\sigma^{(j)})$ time units. Otherwise, we set the expected penalty to be $\penexp(u_j)$. The exact definition of $w_{\product}(\sigma^{(j)})$ can be found in Tab.~\ref{tab:pensim}.

\begin{table*}[!ht]
\centering
\renewcommand{\arraystretch}{1.4}
{
\centering
\begin{tabular}{|c|}
\hline
$
\textcolor{blue}{\pensim(u_j,t_{i,k}+w_{\product}(\sigma^{(j)}))}=\begin{cases}
\pen(u_j,t_{i,k})+\frac{w_{\product}(\sigma^{(j)})}{r} & \text{if }u_j\in \Vis(u_{i,k}), w_{\product}(\sigma^{(j)})\leq h \text{ and }\pen(u_j,t_{i,k})+\frac{w_{\product}(\sigma^{(j)})}{r}\leq 1,\\
\sum \limits_{x=0}^{r} \textcolor{blue}{\mathsf{pst}(\frac{x}{r})} + \textcolor{blue}{\mathsf{pst}(1)}  & \text{if }u_j\in \Vis(u_{i,k}), w_{\product}(\sigma^{(j)})\leq h \text{ and }\pen(u_j,t_{i,k})+\frac{w_{\product}(\sigma^{(j)})}{r}>1,\\
\penexp(u_j) & \text{otherwise.}
\end{cases}
$\\
\hline
$
\textcolor{blue}{\mathsf{pst}(\frac{x}{r})} = \Big( \sum \limits_{y=0}^{z_1} \frac{(z_1-y+z_2+y(r+1))!}{(z_1-y)!\cdot (z_2+y(r+1))!} \cdot (1-p(u_j))^{(z_1-y)}\cdot (p(s))^{(z_2+y(r+1))}\Big)\cdot (1-p(s))\cdot \frac{x}{r}
$\\
$
\text{if }z=w_{\product}(\sigma^{(j)})-(1-\pen(u_j,t_{i,k}))\cdot r -x -1\geq 0, z_1=z \intdiv (r+1), z_2=z \intmod (r+1);
\text{ otherwise if }z<0, \textcolor{blue}{\mathsf{pst}(\frac{x}{r})}=0
$\\
$\textcolor{blue}{\mathsf{pst}(1)} = \sum \limits_{y=0}^{z_3} \frac{(z_3-y+z_4+y(r+1))!}{(z_3-y)!\cdot (z_4+y(r+1))!} \cdot (1-p(u_j))^{(z_3-y)}\cdot (p(s))^{(z_4+y(r+1))}\cdot p(s)
$\\
$
\text{where }z=w_{\product}(\sigma^{(j)})-(1-\pen(u_j,t_{i,k}))\cdot r -1, z_3= z \intdiv (r+1),z_4= z \intmod (r+1)
$\\
\hline
\end{tabular}
}
\caption{The function computing the simulated expected penalty incurred in a state $u_j$ of the run $\sigma$ at the time of its visit $t_{i,k}+w_{\product}(\sigma^{(j)})$ if we are to apply the run $\sigma$ from the current state $u_{i,k}$, $div$ stands for integer division and $mod$ for modulus.}
\label{tab:pensim}
\end{table*}

For a set of states $X\subseteq S_{\product}$, we define a \emph{shortening indicator function} $I_{X}\colon T_{\product}\to \{0,1\}$ such that for $((s_1,q_1),(s_2,q_2))\in T_{\product}$ 
\begin{equation}\label{eq:indicator}
\small
I_{X}\big(((s_1,q_1),(s_2,q_2))\big)=\begin{cases}
1 & \text{if }w_{\product}^*((s_1,q_1),X) \\
& \qquad >w_{\product}^*((s_2,q_2),X),\\
0 & \text{otherwise.}
\end{cases}
\end{equation}
Intuitively, the indicator has value $1$ if the transition leads strictly closer to the set $X$, and $0$ otherwise.

In the first phase of every round, we locally improve the strategy $\stratoff^{\phi}$ computed in Sec.~\ref{subsec:appcoff} that aims to visit an accepting state of the chosen ASCC $\scc$. In each step of the resulting strategy $\straton^{\phi}$, we consider the set $\run_{\phi}(u_{i,k})$ of all finite runs from the current state $u_{i,k}$ that lead to an accepting state from the set $F_{\scc}$ with all transitions shortening in the indicator $I_{F_{\scc}}$ defined according to Eq.~(\ref{eq:indicator}), \ie
\begin{equation*}\footnotesize
\begin{split}
\run_{\phi}(u_{i,k})= \{ & \sigma\in \runfin^{\product}(u_{i,k}) \mid last(\sigma)\in F_{\scc},\\
& \forall 0\leq j\leq |\sigma|-1\colon I_{F_{\scc}}((\sigma(j),\sigma(j+1)))=1\}.
\end{split}
\end{equation*}
Let $\sigma\in \run_{\phi}(u_{i,k})$ be the run that minimizes the optimization function $f$ from Eq.~(\ref{eq:optf}). Then $\straton^{\phi}(\sigma_{\mathrm{all}})=\sigma(1)$. 
Just like in the offline algorithm, the strategy $\straton^{\phi}$ is applied until a state from the set $F_{\scc}$ is visited.

In the second phase, we locally improve the strategy $\stratoff^V$ for the average subgoal computed in Sec.~\ref{subsec:appcoff} to obtain a strategy $\straton^V$. However, the definition of the set of finite runs we choose from changes during the phase. At the beginning of the second phase of the current round $i$, we aim to reach the cycle $\cyc_{\scc}^V=c_0\ldots c_m$ of the ASCC $\scc$ and we use the same idea that is used in the first phase above. To be specific, we define $\straton^V(\sigma_{\mathrm{all}})=\sigma(1)$, where $\sigma$ is the finite run minimizing $f$ from the set
 \begin{equation*}\footnotesize
\begin{split}
\run_{V}(u_{i,k})= \{ & \sigma\in \runfin^{\product}(u_{i,k}) \mid last(\sigma)\in \cyc_{\scc}^V,\\
& \forall 0\leq j\leq |\sigma|-1\colon I_{\cyc_{\scc}^V}((\sigma(j),\sigma(j+1)))=1\}.
\end{split}
\end{equation*}
Once a state $c_a\in \cyc_{\scc}^V$ of the cycle is reached, we continue as follows. Let $c_b\in \cyc_{\scc}^V$ be the first state labeled with $\pisur$ that is visited from $c_a$ if we follow the cycle. 
Until we reach the state $c_b$, the optimal finite run $\sigma$ is chosen from the set 
\begin{equation*}\footnotesize
\begin{split}
\run_{V}(u_{i,k})= \{ & \sigma\in \runfin^{\product}(u_{i,k}) \mid last(\sigma)=c_b,\text{ and }\\
& \forall 0\leq j\leq |\sigma|-1\colon I_{c_b}((\sigma(j),\sigma(j+1)))=1\text{ or }\\
& |\sigma_{c_a\to u_{i,k}}|+|\sigma|\leq  b-a+2 \intmod(m+1)\},
\end{split}
\end{equation*}
where 
$\sigma_{c_a\to u_{i,k}}$ is the finite run already executed in $\product$ from the state $c_a$ to the current state $u_{i,k}$. Intuitively, the set contains every finite run from the current state to the state $c_b$ that either has all transitions shortening in $I_{c_b}$ or the length of the finite run is such that if we were to perform the finite run, the length of the performed run from $c_a$ to $c_b$ would not be longer than following the cycle from $c_a$ to $c_b$. When the state $c_b$ is reached, we restart the above procedure with $c_a=c_b$. The strategy $\straton^V$ is performed until $l_i\geq \max \{j(\tfrac{1}{i}),i\cdot k_i\}$ surveillance cycles are completed in the second phase of the current round $i$, where $k_i$ is the number of steps of the first phase and $j$ is from Prop.~\ref{th:limprob}. We can end the second phase sooner, specifically in any time when we complete a surveillance cycle and the average cumulative penalty per surveillance cycle incurred in the current round is less or equal to $V_{\scc}^*+\tfrac{2}{i}$.


\begin{theorem}\label{th:oncorr}
The strategy $\straton$ projects to a strategy $\strat$ of $\dts$ solving Problem~\ref{pf:appc}.
\end{theorem}

\begin{proof}
First, we prove that Prop.~\ref{th:limprob} holds for the strategy $\straton^V$ as well. This result follows directly from the facts below. The set of finite runs we choose from always contains a finite run induced by the strategy $\stratoff^V$. Once the cycle $\cyc_{\scc}^V$ is reached, the system optimizes the finite run from one surveillance state of the cycle to the next, until it is reached after finite time. Finally, if the strategy $\straton^V$ does not follow $\stratoff^V$, it is only because the chosen finite run provides lower expected average. The correctness of the strategy $\straton$ is now proved analogously to the correctness of the strategy computed offline.
\end{proof}

\begin{proposition}
The strategy $\straton$ is with probability 1 expected to perform in the worst case as good as the strategy $\stratoff$ computed offline. That means, if the average cumulative penalty per surveillance cycle incurred in the so far performed run of the system is lower than the optimal APPC value $V_{\scc}^*$, it will rise slower under the strategy $\straton$ than under the strategy $\stratoff$. On the other hand, if the average cumulative penalty per surveillance cycle incurred in the so far performed run of the system is higher than the optimal APPC value $V_{\scc}^*$, it is expected to decrease faster under the strategy $\straton$ than under the strategy $\stratoff$.
\end{proposition}

\begin{proof}
Follows from the proof of Theorem~\ref{th:oncorr}.
\end{proof}

\subsubsection*{Complexity}

The cardinality of the set of finite runs $\run_{\phi}(u_{i,k})$ grows exponentially with the minimum weight $w_{\product}^*(u_{i,k},F_{\scc})$. Analogously, the same holds for the set of finite runs $\run_{V}(u_{i,k})$ and the set $\cyc_{V}^*$ or one of its surveillance states. To simplify the computations and effectively use the algorithm in real time, one can use the following rule that was also applied in our implementation in Sec.\ref{sec:casestudy}. We put a threshold on the maximum weight of a finite run in $\run_{\phi}(u_{i,k})$ and $\run_{V}(u_{i,k})$. In the second phase of a round, when on the optimal cycle, we optimize the finite run from the state $c_a$ to the next surveillance state on the cycle $c_b$. However, if the weight of the fragment of the cycle from $c_a$ to $c_b$ is too high, we can first optimize the run to some intermediate state $c_b'$. Also, the complexity of one step of the strategy $\straton$ grows exponentially with the user-defined planning horizon $h$. Hence, $h$ should be chosen wisely. One should also keep in mind that the higher the planning horizon, the better local improvement.   



\section{Discussion}\label{sec:discussion}


Every LTL formula $\varphi$ over $\ap$ can be converted to a formula $\phi$ of the form in Eq.~(\ref{eq:ltlsur}) for which it holds that a run of the TS $\dts$ satisfies $\phi$ if and only if it satisfies $\varphi$. The formula is $\phi=\varphi \, \wedge \, \G\F\,\pisur$ where $\pisur \in L(s)$ for every $s\in S$. In that case, Problem~\ref{pf:appc} requires to minimize the expected average penalty incurred per stage. 

The algorithms presented in Sec.~\ref{sec:solution} can be used to correctly solve Problem~\ref{pf:appc} also for the systems with different penalty dynamics than the one defined in Sec.~\ref{sec:pf}.
However, for every state we need to be able to compute the expected value of the penalty in the state, like in Eq.~(\ref{eq:penexp}). 
For the online algorithm we also require that the dynamics of penalties allows to simulate them for a finite number of time
units. More precisely, if we observe the penalty in a state $s\in S$ in time $t$, we can compute the simulated expected value of the penalty in $s$ in every following time unit, up to $h$ time units, based only on the observed value.


The online algorithm from Sec.~\ref{subsec:appcon} is a heuristic. The sets of finite runs $\run_{\phi}(u_{i,k}),\run_V(u_{i,k})$ can be defined differently 
according to the properties of the actual problem. To guarantee the correctness of the strategy $\straton$, the sets must satisfy the following conditions. There always exists a finite run in the set minimizing the optimization function $f$ in Eq.~(\ref{eq:optf}). The definition of the set $\run_{\phi}(u_{i,k})$ guarantees that an accepting state from $F_{\scc}$ is visited after finite number of steps. The definition of $\run_{V}(u_{i,k})$ also guarantees a visit of the cycle $\cyc_{\scc}^V$ in finite time and moreover, Prop.~\ref{th:limprob} holds for the resulting strategy $\straton^V$.


\section{Case study}\label{sec:casestudy}

We implemented the framework developed in this paper for a persistent surveillance robotics example in Java~\cite{tool}. In this section, we report on the simulation results. 

We consider a mobile robot in a grid-like partitioned environment modeled as a TS depicted in Fig.~\ref{subfig:dts}. The robot transports packages between two stocks, marked green in Fig.~\ref{subfig:dts}. The blue state marks the robot's base location. The penalties in states are defined by rate $r=5$ and penalty probability function in Fig.~\ref{subfig:prob}. The visibility range $\vis$ is 6. For example, in Fig.~\ref{subfig:dts} the set $\Vis(s)$ of states visible from the current state $s$, with corresponding penalties, is depicted as the blue-shaded area. We set the planning horizon $h=9$.

\newsavebox{\myimage}

\begin{figure}[t]
\centering
\savebox{\myimage}{\scalebox{0.245}{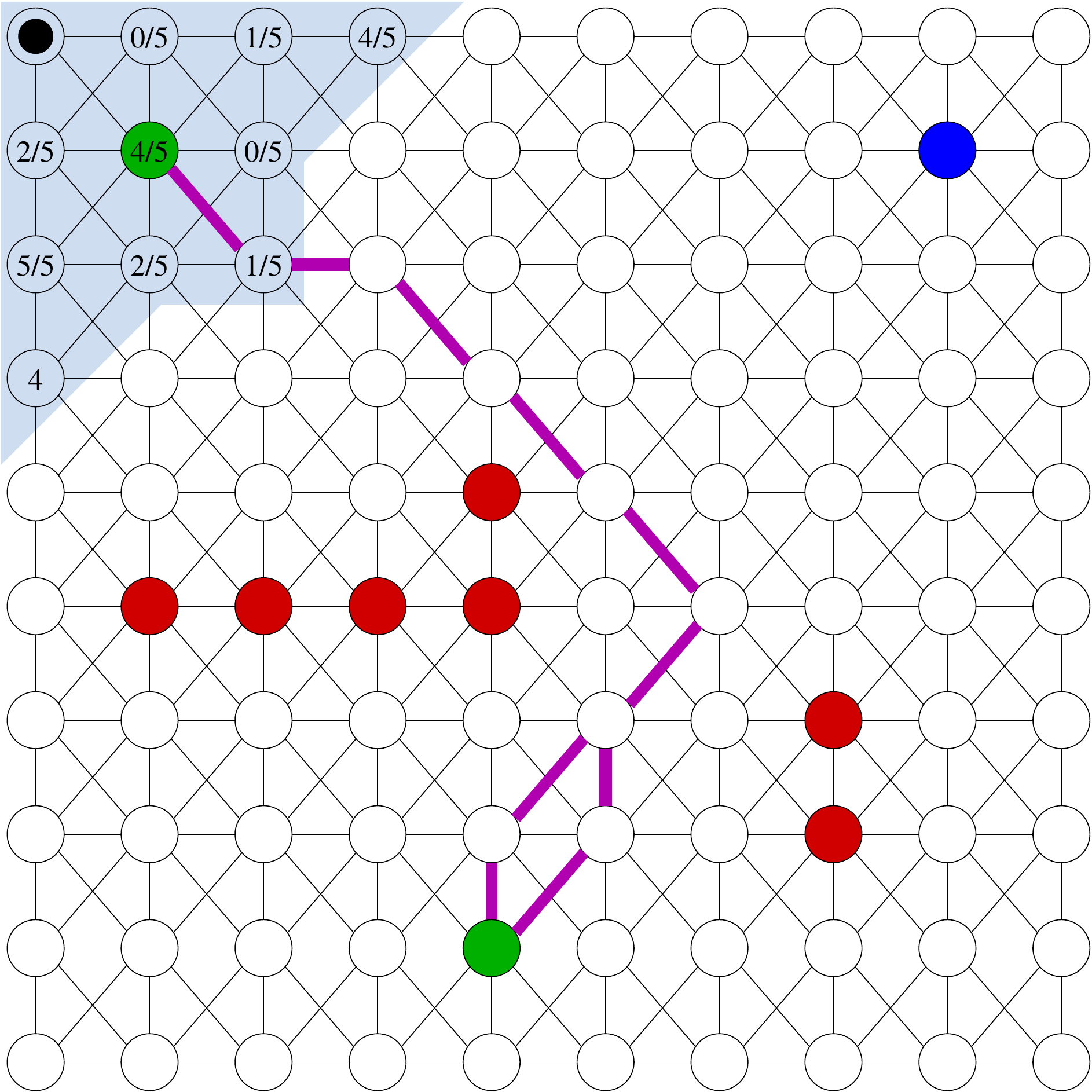}}
\subfloat[]{\label{subfig:dts}\usebox{\myimage}} \quad
\subfloat[]{\label{subfig:prob}\raisebox{\dimexpr.5\ht\myimage-.5\height\relax}{\scalebox{0.3}{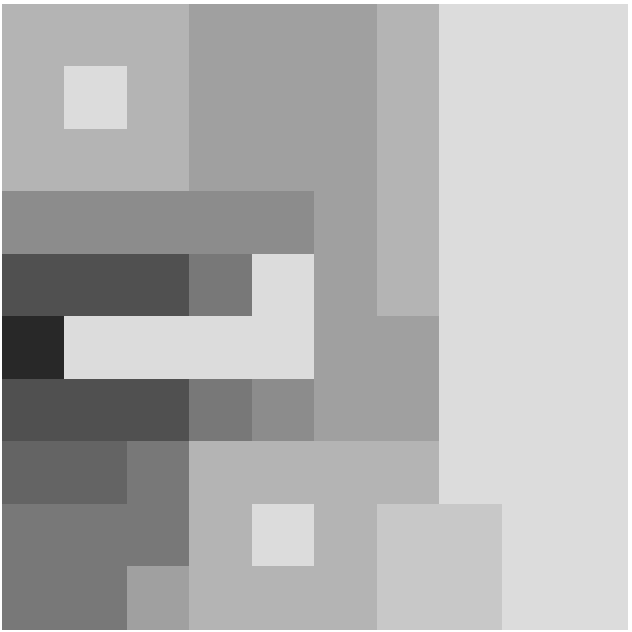}}}
\label{fig:example}
\caption{(a) A TS modeling the robot (black dot) motion in a partitioned environment. 
Two stock locations are in green, 
a base 
is shown in blue, and unsafe locations 
are in red. There is a transition between vertically, horizontally or diagonally neighboring states. 
The weight of a horizontal and vertical transition is 2, for a diagonal transition it is 3. 
(b) The penalty probabilities in states. Darker shade indicates higher probability.}
\end{figure}

The mission for the robot is to transport packages between the two stocks (labeled with propositions $a$, and $b$, respectively) and infinitely many times return to the base (labeled with proposition $c$). The red states in Fig.~\ref{subfig:dts} are dangerous locations (labeled with $u$) which are to be avoided. At the same time, we wish to minimize the cumulative penalty incurred during the transport of a package, \ie the surveillance property $\pisur$ is true in both stock states.
The corresponding LTL formula is
\begin{align*}
& \G\,\big(a \Rightarrow\X\, (\neg a\,\U\,b)\big) \ \wedge \
\G \, \big(b\Rightarrow \X \, (\neg b\,\U\,a )\big)\ \wedge \\
& \G\F\, c \ \wedge \ \G(\neg u) \ \wedge \ \G\,\F\,\pisur,
\end{align*}
and the B\"{u}chi automaton has 10 states.
The cycle providing the minimum expected average cumulative penalty per surveillance cycle is depicted in magenta in Fig.~\ref{subfig:dts} and the optimal APPC value is $5.4$. 

We ran both offline and online algorithm for multiple rounds starting from the base state. In Fig.~\ref{fig:plots} we report on the results for 20 rounds, for more results see~\cite{tool}. As illustrated in Fig.~\ref{fig:plots}a, the average cumulative penalty per surveillance cycle incurred in the run induced by the offline strategy is above the optimal value and converges to it fairly fast. For the run induced by the online strategy, the average is significantly below the minimum APPC value due to the local improvement based on local sensing. On the other hand, Fig.~\ref{fig:plots}b shows the average cumulative penalty per surveillance cycle incurred in each round separately. The number of surveillance cycles performed in the second phase of every round $i$ of the offline strategy was less than $i\cdot k_i$, \ie the second phase always ended due to the fact that the average incurred in the round was below the threshold $V^*_{\scc}+\tfrac{2}{i}$. The maximum number of surveillance cycles performed in the second phase of a round was $7$. The same is true for the online strategy and the maximum number of surveillance cycles in the second phase of a round was $3$. For both algorithms, the number of surveillance cycles in the second phase of a round does not evolve monotonically, rather randomly. Hence we conclude that in every round $i$ we unlikely need to compute the value $j(\tfrac{1}{i})$. 

\begin{figure}[t]
\centering
\begin{tabular}{l l}
{\footnotesize\raisebox{1cm}{(a)}} & \includegraphics[width=0.89\linewidth]{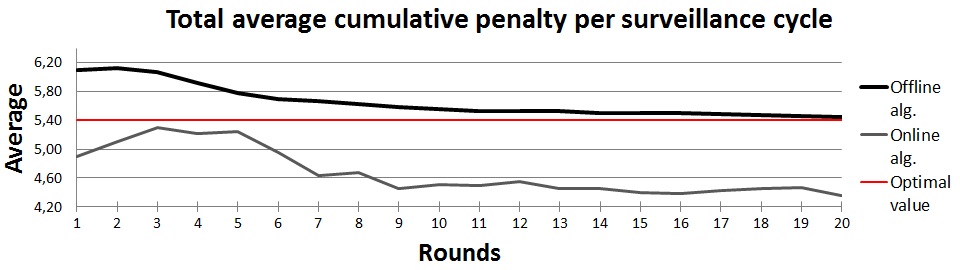} \\
{\footnotesize\raisebox{1cm}{(b)}} & \includegraphics[width=0.89\linewidth]{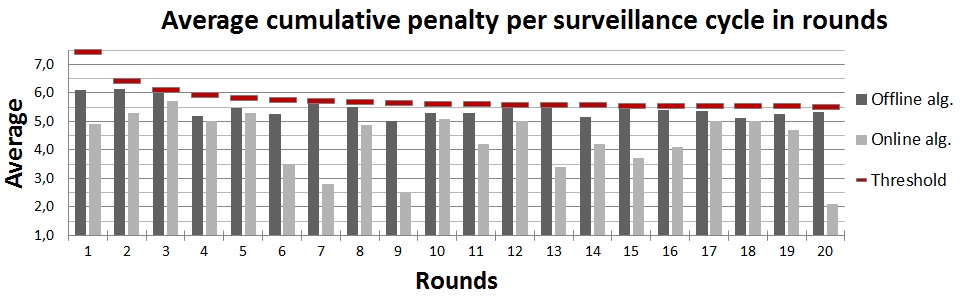}
\end{tabular}
\caption{(a) The average cumulative penalty per surveillance cycle incurred during the runs, shown at the end of each round. The red line marks the optimal APPC value. (b) The average cumulative penalty per surveillance cycle incurred in every round. The red bars indicate the threshold $V_{\scc}^*+\tfrac{2}{i}$.}
\label{fig:plots}
\end{figure}



\bibliographystyle{IEEEtran}
\bibliography{dts_acps_acpc_rh}

\end{document}

%% file: reduction.pdf_t
\begin{picture}(0,0)%
\includegraphics{reduction.pdf}%
\end{picture}%
\setlength{\unitlength}{3947sp}%
\begingroup\makeatletter\ifx\SetFigFont\undefined%
\gdef\SetFigFont#1#2#3#4#5{%
  \reset@font\fontsize{#1}{#2pt}%
  \fontfamily{#3}\fontseries{#4}\fontshape{#5}%
  \selectfont}%
\fi\endgroup%
\begin{picture}(7816,2019)(1043,-2023)
\put(2701,-211){\makebox(0,0)[b]{\smash{{\SetFigFont{14}{16.8}{\rmdefault}{\mddefault}{\updefault}{\color[rgb]{0,0,0}$\seq_1$}%
}}}}
\put(7201,-211){\makebox(0,0)[b]{\smash{{\SetFigFont{14}{16.8}{\rmdefault}{\mddefault}{\updefault}{\color[rgb]{0,0,0}$\seq_8$}%
}}}}
\put(2701,-586){\makebox(0,0)[b]{\smash{{\SetFigFont{14}{16.8}{\rmdefault}{\mddefault}{\updefault}{\color[rgb]{0,0,0}$\seq_3$}%
}}}}
\put(2701,-1936){\makebox(0,0)[b]{\smash{{\SetFigFont{14}{16.8}{\rmdefault}{\mddefault}{\updefault}{\color[rgb]{0,0,0}$\seq_7$}%
}}}}
\put(7201,-1936){\makebox(0,0)[b]{\smash{{\SetFigFont{14}{16.8}{\rmdefault}{\mddefault}{\updefault}{\color[rgb]{0,0,0}$\seq_9$}%
}}}}
\put(3826,-886){\makebox(0,0)[b]{\smash{{\SetFigFont{14}{16.8}{\rmdefault}{\mddefault}{\updefault}{\color[rgb]{0,0,0}$\seq_4$}%
}}}}
\put(3826,-1336){\makebox(0,0)[b]{\smash{{\SetFigFont{14}{16.8}{\rmdefault}{\mddefault}{\updefault}{\color[rgb]{0,0,0}$\seq_6$}%
}}}}
\put(1576,-886){\makebox(0,0)[b]{\smash{{\SetFigFont{14}{16.8}{\rmdefault}{\mddefault}{\updefault}{\color[rgb]{0,0,0}$\seq_2$}%
}}}}
\put(1576,-1336){\makebox(0,0)[b]{\smash{{\SetFigFont{14}{16.8}{\rmdefault}{\mddefault}{\updefault}{\color[rgb]{0,0,0}$\seq_5$}%
}}}}
\put(8626,-886){\makebox(0,0)[b]{\smash{{\SetFigFont{14}{16.8}{\rmdefault}{\mddefault}{\updefault}{\color[rgb]{0,0,0}$\seq_5.\seq_4$}%
}}}}
\put(8626,-1336){\makebox(0,0)[b]{\smash{{\SetFigFont{14}{16.8}{\rmdefault}{\mddefault}{\updefault}{\color[rgb]{0,0,0}$\seq_2.\seq_6$}%
}}}}
\end{picture}%

%% file: example.pdf_t
\begin{picture}(0,0)%
\includegraphics{example.pdf}%
\end{picture}%
\setlength{\unitlength}{4144sp}%
\begingroup\makeatletter\ifx\SetFigFont\undefined%
\gdef\SetFigFont#1#2#3#4#5{%
  \reset@font\fontsize{#1}{#2pt}%
  \fontfamily{#3}\fontseries{#4}\fontshape{#5}%
  \selectfont}%
\fi\endgroup%
\begin{picture}(8615,8615)(1069,-8394)
\end{picture}%

%% file: exampleprob.pdf_t
\begin{picture}(0,0)%
\includegraphics{exampleprob.pdf}%
\end{picture}%
\setlength{\unitlength}{3947sp}%
\begingroup\makeatletter\ifx\SetFigFont\undefined%
\gdef\SetFigFont#1#2#3#4#5{%
  \reset@font\fontsize{#1}{#2pt}%
  \fontfamily{#3}\fontseries{#4}\fontshape{#5}%
  \selectfont}%
\fi\endgroup%
\begin{picture}(3024,3024)(1189,-3373)
\end{picture}%